\documentclass{article}

\usepackage[T1]{fontenc}
\usepackage[utf8]{inputenc}
\usepackage{amsmath}
\usepackage{amssymb}
\usepackage{amsthm}
\usepackage{graphicx}
\usepackage{svg}
\usepackage{natbib}
\usepackage{url}

\newtheorem{theorem}{Theorem}
\newtheorem{prop}{Proposition}

\begin{document}

    \title{
      Sliding Mode Control and Subspace Stabilization Methodology 
      for the Orbital Stabilization of
      Periodic Trajectories
    }

    \author{
      Maksim Surov,
      Leonid Freidovich
    }

    \maketitle
    \thispagestyle{empty}
    \pagestyle{empty}

    \begin{abstract}
      This paper presents a combined sliding-mode control and subspace stabilization methodology for orbital stabilization of periodic trajectories in underactuated mechanical systems with one degree of underactuation. The approach starts with partial feedback linearization and stabilization. Then, transverse linearization along the reference orbit is computed, resulting in a periodic linear time-varying system with a stable subspace. Sliding-mode control drives trajectories toward this subspace. The proposed design avoids solving computationally intensive periodic LQR problems and improves robustness to matched disturbances. The methodology is validated through experiments on the Butterfly robot.
    \end{abstract}


  \maketitle

  \section{Introduction}
    The problem of orbital stabilization of periodic trajectories has been addressed in a series of publications:~\cite{Banaszuk1995,Shiriaev2005,Maggiore2013,Finet2015,Saetre2021,Kant2020,Surov2020}.
    Many of these works, e.g.,~\cite{Banaszuk1995,Shiriaev2005,Finet2015,Surov2020}, employ the transverse linearization approach, which approximates the dynamics near a reference periodic orbit by a linear time-varying (LTV) system with periodic coefficients.
    As shown in~\cite{Shiriaev2005,Shiriaev2010}, a feedback designed to stabilize the trivial solution of this auxiliary LTV system can be used to construct a control law that stabilizes the orbit of the original nonlinear system. Under the mild assumption of controllability of the LTV system over one period, the LQR approach can be used to design the feedback. The practical effectiveness of this method was demonstrated in experiments with real robotic systems in~\cite{Freidovich2008,Manchester2011,Surov2015}.

    A substantially different stabilization method for the LTV system was proposed in~\cite{Saetre2021}, where the authors developed an alternative scheme combining Floquet theory with sliding-mode control. Following this line of work, we show that a specific feedback linearization of the transverse dynamics yields an LTV system endowed with a stable invariant subspace. In this setting, the control objective reduces to driving all trajectories into the stable subspace, which is achieved via sliding-mode-based control. This method does not require solving the computationally demanding periodic LQR problem. 
		
    The effectiveness of the proposed method is confirmed experimentally by stabilizing a periodic motion of the Butterfly robot, implemented on a real hardware platform.

  \section{Problem formulation}
    We consider a class of controlled mechanical systems described by
    the Euler-Lagrange equations
    \begin{equation}
      M(q)\ddot{q}+C\left(q,\dot{q}\right)\dot{q}+G(q)=F(q)u,\label{eq:mechanical-system}
    \end{equation}
    where $q\in\mathbb{R}^{2}$ denotes the generalized coordinates, $\dot{q}\equiv\frac{dq}{dt}$
    the generalized velocities, and $u\in\mathbb{R}$ the control input.
    The matrix $M(q)\in\mathbb{R}^{2\times2}$ is symmetric
    and positive definite, $C\left(q,\dot{q}\right)\in\mathbb{R}^{2\times2}$
    is linear in $\dot{q}$, $F(q)\in\mathbb{R}^{2\times1}$
    is a nonzero input vector. All functions are assumed to be continuously
    differentiable. The system is underactuated with underactuation
    degree one. 

    Let us introduce state variables $x\triangleq\left(q,\dot{q}\right)\in\mathbb{R}^{4}$ and
    suppose that, in the state-space, the system admits a non-trivial $T$-periodic trajectory
    \[
      x_{*}(t)\triangleq\bigl(q_{*}(t),\dot{q}_{*}(t)\bigr)\in\mathbb{R}^{4}
    \]
    associated with a virtual holonomic constraint (VHC) of the form $h(q)=0$,
    where $h\in C^{2}\left(\mathbb{R}^{2},\mathbb{R}\right)$ satisfies
    $\frac{\partial h(q)}{\partial q}\ne0$ for all $q$ in
    a neighborhood of $q_{*}(t)$, and $h\left(q_{*}(t)\right)\equiv0$ (see, e.g., \cite{Beghin1922,Urabe1967,Aoustin1999,Canudas2002}).
    We refer to 
    \[
      x_{*} \triangleq \left\{ x\in\mathbb{R}^{4}\mid x=x_{*}(t),\,t\in\mathbb{R}\right\} 
    \]
    as the orbit of the trajectory $x_{*}(t)$. The tubular
    $\varepsilon$-neighborhood of this orbit is defined as
    \[
      U_{\varepsilon} \triangleq \left\{ x\in\mathbb{R}^{4}\mid\mathrm{dist}\left(x,x_{*}\right)<\varepsilon\right\} .
    \]
    
    We address the problem of designing a feedback law $u\left(x\right)$
    that ensures the asymptotic orbital stability of the periodic trajectory $x_{*}(t)$ (see, e.g., \cite{Khalil2002,Canudas2002}). In contrast to the existing approaches proposed in~\cite{Shiriaev2005,Maggiore2013,Kant2020}, our solution is based on sliding-mode control methodology and further develops the ideas introduced in~\cite{FreidovichGusev2018,Saetre2021}.

  \subsection*{Notation}
    For a function $g:\mathbb{R}^{2}\to\mathbb{R}$ and a vector field
    $f:\mathbb{R}^{2}\to\mathbb{R}^{2}$, the Lie derivative of $g$ along
    $f$ is defined as $L_{f}g\equiv\frac{\partial g}{\partial q}f$.
    Higher-order Lie derivatives are defined recursively as $L_{f}^{2}g\equiv L_{f}L_{f}g$.
    
  \section{Orbital Stabilization Approach}
  \label{sec:Orbital-Stabilization-Approach}
    Before presenting the main result, we briefly outline the transverse linearization scheme for system~\eqref{eq:mechanical-system}, following the steps presented in~\cite{Shiriaev2005}.

  \subsection{Feedback Transformation for VHC Stabilization}
    For system~\eqref{eq:mechanical-system}, we introduce a new control variable $w$, defined as
    \begin{align}
    \label{eq:control-transform}
      w & \triangleq \frac{\partial h(q)}{\partial q}M^{-1}(q)\Bigl(-C\left(q,\dot{q}\right)\dot{q}-G(q)+F(q)u\Bigr) + \nonumber \\
        & \qquad L_{\dot{q}}^{2}h(q)
        +\nu_{1}h(q)+\nu_{2}L_{\dot{q}}h(q),
    \end{align}
    where $\nu_{1},\nu_{2}>0$ are design parameters. 
    Assuming that the VHC satisfies the regularity condition
    \[
      \frac{\partial h(q)}{\partial q}M^{-1}(q)F(q)\ne0
    \]
    for all $q$ in a neighborhood of $q_{*}(t)$, justifies validity of this transformation.
    After applying transformation~\eqref{eq:control-transform}, the dynamics~\eqref{eq:mechanical-system} can be written in the normal form
    \begin{equation}
    \label{eq:dynamics-normal-form}
      \dot{x}=f\left(x\right)+g\left(x\right)w.
    \end{equation}
    It is straightforward to verify that the function $x_{*}(t)$ is a solution of the unforced system
    \[
      \frac{d}{dt}x_{*}(t)\equiv f\bigl(x_{*}(t)\bigr).
    \]
    Moreover, the unforced system exhibits an important property stated in
    \begin{prop}\textit{
    \label{prop:stability-of-y}
      Consider the system
      \[
        \dot{x}=f\left(x\right)
      \]
      obtained from~\eqref{eq:dynamics-normal-form} by setting $w=0$
      and feedback transformation parameters $\nu_{1},\nu_{2}>0$. Define
      \[
        y\left(x\right) \triangleq \left(\begin{array}{c}
        h(q)\\
        L_{\dot{q}}h(q)
        \end{array}\right)\in\mathbb{R}^{2}.
      \]
      Then there exist a positive definite matrix $P\in\mathbb{R}^{2\times2}$
      and a constant $\alpha>0$ such that the quadratic form
      \[
        V_{y}\left(x\right) \triangleq y^{\top}\!(x)\,P\,y(x)
      \]
      decays exponentially along all solutions of the unforced system within
      $U_{\varepsilon}$, that is,
      \[
        \frac{d}{dt}V_{y}\bigl(x(t)\bigr) \leq -\alpha V_{y}\bigl(x(t)\bigr).
      \]
    }\end{prop}
		
    See proof in Appendix~\ref{sec:proof-stability-of-y}.
    
    In what follows, we demonstrate that Proposition~\ref{prop:stability-of-y}
    endows a specific structure in the linearization of dynamics~\eqref{eq:dynamics-normal-form} along a periodic orbit.
    
  \subsection{Transverse Dynamics}
  \label{sub:transverse-dynamics}
    Consider a coordinate transformation
    \[
      \tau = \tau(x) \in \mathbb{R} \quad\text{and}\quad \xi = \xi(x) \in \mathbb{R}^3
    \]
    defined in $U_{\varepsilon}$, such that:
    \begin{enumerate}
      \item 
        The transverse coordinates $\xi$ vanish on the reference
        trajectory, i.e., $\xi\left(x_{*}(t)\right) = 0$ for all
        $t\in\mathbb{R}$.
      \item 
        The projection variable $\tau$, when evaluated on the reference trajectory, $\tau_*(t) = \tau\left(x_*(t)\right)$, satisfies
        \[
          \frac{d\tau_*(t)}{dt}\geq\mathrm{const}>0,\quad t\in\left[0,T\right)
        \]
        ensuring strict monotonic increase over one period. 
        Moreover, the function $\tau_*(t)$ extends continuously beyond the interval via
        \begin{equation}
        \label{eq:extending-tau}
          \tau_*(t + nT) = \tau_*(t) + n \, T_{\tau}\quad n\in\mathbb{Z},
        \end{equation}
        for some $T_{\tau}>0$.
      \item 
        The transformation $x \mapsto\left(\tau,\xi\right)$ is a diffeomorphism
        on $U_{\varepsilon}$.
    \end{enumerate}
    In the new coordinates, the dynamics~\eqref{eq:dynamics-normal-form}
    can be represented within $U_{\varepsilon}$ as (see, e.g.,~\cite{Leonov1995,Banaszuk1995,Finet2015}):
    \begin{equation}
    \label{eq:transverse-dynamics}
      \frac{d\xi}{d\tau}=A\left(\tau\right)\xi+B\left(\tau\right)w+o\left(\tau,\xi,w\right),
    \end{equation}
    where $A\left(\tau\right)\in\mathbb{R}^{3\times3}$ and $B\left(\tau\right)\in\mathbb{R}^{3\times1}$ are $T_{\tau}$-periodic matrix functions and $o\left(\tau,\xi,w\right)$
    collects bilinear, quadratic, and higher-order terms in $\xi$ and
    $w$. 
    A feedback law $w\left(\tau,\xi\right)$ that asymptotically stabilizes the trivial solution $\xi=0$ of the transverse dynamics~\eqref{eq:transverse-dynamics},
    together with the transformation~\eqref{eq:control-transform}, allows the construction of a control law $u\left(x\right)$ for the original nonlinear system~\eqref{eq:mechanical-system}. Using the arguments in~\cite{Leonov1998,Demidovich1968}, the resulting feedback $u\left(x\right)$ guarantees local asymptotic orbital stability of the periodic trajectory $x_{*}(t)$.
    
    A common approach in a series of publications,~\cite{Shiriaev2005,Freidovich2008,Shiriaev2010,Ahmed2013,Surov2015}, for stabilizing the system~\eqref{eq:transverse-dynamics} is to define $w\left(\tau,\xi\right)=K\left(\tau\right)\xi$,
    where the periodic feedback gains $K\left(\tau\right)\in\mathbb{R}^{1\times3}$ are computed as the solution of the LQR problem for the linearized system
    \begin{equation}
    \label{eq:linearized-transverse-dynamics}
      \frac{d\xi}{d\tau}=A\left(\tau\right)\xi+B\left(\tau\right)w,
    \end{equation}
    provided the system is controllable over one period.
    In contrast, we propose an alternative feedback design for system~\eqref{eq:linearized-transverse-dynamics} based on the sliding-mode control methodology. 
    While the LQR approach is applicable to a wide class of controllable linear systems, our method exploits specific properties of the matrix $A\left(\tau\right)$, which are discussed below.

  \subsection{Stable Floquet Subspace of Transverse Linearization}
    The property of dynamics~\eqref{eq:dynamics-normal-form} established
    in Proposition~\ref{prop:stability-of-y} implies the existence of
    a stable invariant subspace for the unforced linearized transverse
    dynamics~\eqref{eq:linearized-transverse-dynamics}, which is formally defined in
    \begin{prop}\textit{
    \label{prop:stable-subspace}
      Consider system~\eqref{eq:dynamics-normal-form}
      with $w=0$, assuming the conditions of Proposition~\ref{prop:stability-of-y} are satisfied. Let its transverse linearization be given by the linear time-varying system
      \begin{equation}
      \label{eq:unforced-linear-equation}
        \frac{d\xi}{d\tau}=A\left(\tau\right)\xi,
      \end{equation}
      where $A\left(\tau\right)\in\mathbb{R}^{3\times3}$ is $T_{\tau}$-periodic.
      Assume that the monodromy matrix of~\eqref{eq:unforced-linear-equation} is nondefective.
      Then, equation~\eqref{eq:unforced-linear-equation} admits two stable Floquet multipliers $\left|\mu_{1,2}\right|<1$,
      and there exist two linearly independent Floquet solutions $\xi=e_{1}\left(\tau\right)$
      and $\xi=e_{2}\left(\tau\right)$ associated with $\mu_{1}$ and $\mu_{2}$, respectively\footnote{
        The definition and properties of Floquet multipliers and solutions are given in~\cite{Yakubovich1975}.
      }.
      These solutions span the two-dimensional stable invariant subspace
      \[
        \mathcal{S}_{\tau}
        \triangleq
        \mathrm{span}\left\{e_{1}\left(\tau\right),e_{2}\left(\tau\right)\right\},
      \]
      defined in the extended phase space $\left(\tau,\xi\right)$.
    }\end{prop}
		
    See proof in Appendix~\ref{sec:appendix-prop-2}.

    \begin{prop}\textit{
    \label{prop:definition-of-n}
      Consider two linearly independent solutions $e_1(\tau)$ and $e_2(\tau)$ of equation~\eqref{eq:unforced-linear-equation}.
      There exists a continuously differentiable unit vector function $n(\tau) :\, \mathbb{R} \to S^2$ such that 
      \[
        n^{\top}\!(\tau) e_1(\tau) = n^{\top}\!(\tau) e_2(\tau) = 0 \quad
        \text{and} \quad
        \Vert n(\tau) \Vert = 1 \quad \forall \, \tau.
      \]
      The function $n(\tau)$ satisfies the differential equation
      \begin{equation}
      \label{eq:ode-for-n}
        \frac{dn}{d\tau} = -\left(I - n\, n^\top\right) A^\top\!(\tau)\, n.
      \end{equation}
      This construction provides an alternative characterization of the invariant subspace $\mathcal{S}_\tau$ spanned by $e_1(\tau)$ and $e_2(\tau)$:
      \[
        \mathcal{S}_\tau \equiv \left\{ \xi \in \mathbb{R}^3 \mid n^\top(\tau) \xi = 0 \right\}.
      \]
    }\end{prop}
    See proof in Appendix~\ref{proof:definition-of-n}.
    While the stable subspace $\mathcal{S}_\tau$ can be computed directly using Floquet theory, see~\cite{Yakubovich1975}, equation~\eqref{eq:ode-for-n} provides an alternative approach.
    In particular, if $|\mu_3| > \max(|\mu_1|, |\mu_2|)$, 
    then the solution of the initial value problem~\eqref{eq:ode-for-n} with \textit{ almost any} $n(0)\in S^2$ converges exponentially fast, when integrated backward in time, to the vector function $n(\tau)$ defining the stable invariant subspace $\mathcal{S}_\tau$.
    By \textit{almost any}, we mean that there exists a subset of measure zero on the sphere consisting of initial conditions for which convergence does not occur.
    This behavior follows directly from the properties of the adjoint linear system; a formal proof is omitted.

    The stable subspace $\mathcal{S}_{\tau}$ plays a key role in the construction of the feedback law. To stabilize the transverse variables $\xi$, and consequently the orbit $x_{*}$, it suffices to design a feedback law that drives $\xi$ toward this stable subspace. This principle underlies the control design presented in the following section.

  \subsection{Sliding-Mode Control for Transverse Linearization}
    Below, we state a theorem proposing a control law for a periodic LTV system satisfying the conditions of Proposition~\ref{prop:stable-subspace}.
    \begin{theorem}
    \label{thm:ltv-sliding-feedback}
    \textit{
      Consider the LTV system
      \begin{equation}
      \label{eq:ltv-system}
        \frac{d\xi}{d\tau} = A(\tau)\xi + B(\tau)w,
      \end{equation}
      where $A(\tau) \in \mathbb{R}^{3\times3}$, $B(\tau) \in \mathbb{R}^{3\times1}$ are continuous and $T_{\tau}$-periodic functions.
      Suppose that
      \begin{enumerate}
      \item 
        The unforced system
        \begin{equation}
        \label{eq:ltv-unforced}
          \frac{d\xi}{d\tau} = A(\tau)\xi
        \end{equation}
        admits a two-dimensional stable invariant subspace
        \[
          \mathcal{S}_{\tau}
          \triangleq\left\{ \xi\in\mathbb{R}^3\mid n^{\top}(\tau)\,\xi=0\right\},
        \]
        where $n\left(\tau\right):\mathbb{R}\to S^{2}$ is continuous. That is, 
        any solution originating at $\tau=\tau_o$ in $\xi(\tau_o) = \xi_{o}\in\mathcal{S}_{\tau_{o}}$
        remains in $\mathcal{S}_{\tau}$ for all $\tau\geq\tau_{o}$, and
        satisfies $\xi\left(\tau\right) \to 0$ as $\tau\to\infty.$
      \item 
        The scalar function
        \[
          b(\tau)\triangleq n^{\top}\!\left(\tau\right)B\left(\tau\right)
        \]
        has finitely many isolated zeros in $\left[0,T_\tau\right]$, and each
        zero $\tau_{z}$ is simple, i.e., 
        $\frac{d}{d\tau}b(\tau_z)\ne0$. 
      \end{enumerate}
      Define the scalar variable
      \[
        s \triangleq n^{\top}\!(\tau)\,\xi
      \]
      and consider the feedback law
      \begin{equation}
      \label{eq:sliding-feedback}
        w\bigl(\tau,\xi\bigr) =-\sigma\bigl(b(\tau)\bigr) \bigl(k_{1}\,\mathrm{sign}(s) +k_{2}\,s\bigr),
      \end{equation}
      where $k_{1}>0$,
      \begin{equation}
      \label{eq:inequality-for-k2}
        k_{2}>\frac{\int_{0}^{T_\tau}n^{\top}\!\left(\tau\right)A\left(\tau\right)n\left(\tau\right)dv}
        {\int_{0}^{T_\tau}b(\tau)\,\sigma\bigl(b(\tau)\bigr)\,dv}
      \end{equation}
      and $\sigma(b)\triangleq\frac{b}{\left|b\right|+\varepsilon}$ with some $\varepsilon>0$ is a sigmoid function\footnote{
        The choice $\sigma\left(b\right) = \frac{b}{\left|b\right|+\varepsilon}$ is not unique; any smooth, odd function with a single zero at $b=0$
        and strictly monotone near the origin can be used in feedback.
      }.
      Then, the origin $\xi=0$ of the closed-loop system~(\ref{eq:ltv-system},\ref{eq:sliding-feedback}) is globally asymptotically stable.
    }\end{theorem}

    \begin{proof}
      We begin by computing the dynamics of the variable $s$, which represents the signed distance to the subspace $\mathcal{S}_\tau$, for~\eqref{eq:ltv-system}:
      \[
        \frac{ds}{d\tau} = \frac{dn^{\top}\!(\tau)}{d\tau}\xi+n^{\top}\!(\tau)A(\tau) \,\xi+b(\tau) \,w.
      \]
      Using properties of $n(\tau)$ established in Proposition~\ref{prop:definition-of-n}, this simplifies to
      \begin{equation}
      \label{eq:dot-s}
        \frac{ds}{d\tau} = n^{\top}\! (\tau) A(\tau) n(\tau) s + b(\tau)w.
      \end{equation}
      Under the control law~\eqref{eq:sliding-feedback}, the closed-loop dynamics take the form
      \begin{equation}
      \label{eq:sliding-dynamics}
        \frac{ds}{d\tau} = \left(n^{\top} A \, n - k_{2}b\sigma(b)\right)s - k_1 b \sigma(b) \mathrm{sign}\left(s\right)
      \end{equation}
      Hereafter, the argument $\tau$ is omitted for brevity.

      Let us demonstrate that a trajectory, originating at any $\xi(\tau_o)=\xi_o$, reaches $\mathcal{S}_\tau$ in finite time. Equivalently, the function $s=s\bigl(\tau,\xi(\tau)\bigr) = n^\top\!(\tau)\xi$, evaluated along the trajectory, attains the value $0$ in finite time.
      To this end, consider the evolution of the non-negative function $V(\tau)=s^2$:
      \begin{align*}
      \label{eq:slidingV-dynamics}
        \frac{dV}{d\tau} & =2\bigl(n^{\top}A\,n-k_{2}b\sigma(b)\bigr)V
         -2k_{1}b\sigma(b)\sqrt{V}\nonumber \\
          &\le 2\bigl(n^{\top}A\,n-k_{2}b\sigma(b)\bigr)V
      \end{align*}
      Let
      \[
        \alpha=k_{2}\int_{\tau_{o}}^{\tau_{o}+T_\tau}b\sigma(b)d\tau-\int_{\tau_{o}}^{\tau_{o}+T_\tau}n^{\top}An\,d\tau.
      \]
      Since $b(\tau)$ has finitely many isolated zeros over one period and $b(\tau)\sigma(b(\tau))\geq0$, it follows that
      \[
        \int_{\tau_{o}}^{\tau_{o}+T_\tau}b \sigma(b )d\tau>0.
      \]
      Using inequality~\eqref{eq:inequality-for-k2} for the coefficient $k_{2}$, we conclude that $\alpha>0$. Then, by the comparison principle (or the Grönwall-Bellman inequality), we obtain
      \[
        V(\tau_0+T_\tau)\le V(\tau_0)\,e^{-2\alpha T_\tau}.
      \]
      This inequality implies that $V$ and $|s|$ decay exponentially or faster.

      Consider the case $s_o = n^\top\!(\tau_o) \xi_o > 0$ (the case $s_o<0$ is treated similarly). 
      From the general solution of~\eqref{eq:sliding-dynamics}, we obtain that $s$ reaches zero at some $\tau_{r}$ satisfying
      \begin{align*}
        & \int_{\tau_{o}}^{\tau_{o}+\tau_{r}}b\,\sigma(b)\exp\left(\int_{\tau_{0}}^{v}k_{2}b\,\sigma(b)-n^{\top}\!A n\,dw\right)dv=\frac{\left|s_{o}\right|}{k_{1}}.
      \end{align*}
      If $k_2$ satisfies~\eqref{eq:inequality-for-k2}, then the integral 
      \[
        I\left(\tau_{o},\tau\right) \triangleq \int_{\tau_{o}}^{\tau}\left(k_{2}b\,\sigma(b)-n^{\top}\!A n\right)dw
      \]
      admits the lower bound:
      \[
        m \triangleq \inf_{\tau\geq\tau_{0}}I\left(\tau_{o},\tau\right)=\inf_{\tau\in\left[\tau_{0},\tau_{0}+T\right]}I\left(\tau_{0},\tau\right)\in\mathbb{R}.
      \]
      Therefore, the reaching time $\tau_{r}$ is majorated by the inequality
      \[
        \int_{\tau_{o}}^{\tau_{o}+\tau_{r}}b\,\sigma(b)\,dv\leq\frac{\left|s_{o}\right|}{k_{1}\exp\left(m\right)}.
      \]
      Setting $d \triangleq \int_{0}^{T}n^{\top}\left(b\,\sigma(b)\right)dv>0$, we obtain the explicit estimate:
      \begin{equation*}
      \label{eq:reaching-time}
        \tau_{r}\leq T_\tau\left(1+\left\lceil \frac{\left|s_{o}\right|}{k_{1}\exp\left(m\right)d}\right\rceil \right)<\infty.
      \end{equation*}

      Next, using the equivalent control principle, we show that the dynamics on $\mathcal{S}_\tau$ are governed by~\eqref{eq:ltv-unforced}  and are therefore stable by assumption~(1) of the theorem.
      If $w$ were continuously differentiable, we would have
      \begin{equation}
      \label{eq:ddot-s}
        \frac{d^{2}}{d\tau^{2}}s=\frac{d(n^{\top}An)}{d\tau}s+n^{\top}An\frac{ds}{d\tau}
        +\frac{db}{d\tau}w+b\frac{d}{d\tau}w.
      \end{equation}
      Equation~\eqref{eq:dot-s} implies that for any $\tau$ such that $s=0$, $\frac{d}{d\tau}s=0$, and $b(\tau)\ne0$ we have $w=0$. 
      Furthermore, from~\eqref{eq:ddot-s} we see that for any $\tau=\tau_z$ satisfying $s=0$, $\frac{d}{d\tau}s=0$, $\frac{d^2}{d\tau^2}s=0$, $b(\tau_z)=0$, and $\frac{d}{d\tau}b(\tau_z)\ne0$ we also obtain $w=0$. 
      Therefore, solutions of~\eqref{eq:ltv-system} with $w=0$ and initial conditions on the subspace~$S_\tau$ remain on the subspace and decay exponentially.

      Combining the finite-time reaching result with the invariance and asymptotic stability of the on-manifold dynamics, together with the standard fact that linear systems do not exhibit finite-time escape, yields global asymptotic stability of the origin (understanding solutions of the discontinuous differential equations defining the closed-loop system in the sense of Filippov, see~\cite{Utkin1992}).
      \qed
    \end{proof}

  \subsection{Control Law for Nonlinear System}
    The control law~\eqref{eq:sliding-feedback} designed for the transverse
    dynamics induces a control input for the original system~\eqref{eq:mechanical-system}
    through the inverse of the transformation~\eqref{eq:control-transform}:
    \begin{align}
    \label{eq:nonlinear-control}
      u\left(q,\dot{q},w\right)= & \frac{1}{\frac{\partial h}{\partial q}M^{-1}F}\biggl(w-\nu_{1}h-\nu_{2}L_{\dot{q}}h\\
       & \quad-L_{\dot{q}}^{2}h+\frac{\partial h}{\partial q}M^{-1}\left(C\dot{q}+G\right)\biggr), \nonumber
    \end{align}
    where the functions arguments are omitted for brevity.
    Although a full theoretical proof of orbital asymptotic stability
    for the nonlinear system with the proposed discontinuous feedback
    remains a subject of further rigorous analysis, the effectiveness
    of the method is illustrated by the experimental results presented
    in the following section.
    
  \section{Butterfly Robot Example}
    For demonstrative purposes, we consider the stabilization of a periodic motion of the Butterfly robot introduced in~\cite{Lynch1998}. The experimental platform represents a planar mechanical system consisting of a figure-eight-shaped link mounted on a fixed base through a single revolute joint. The link rotates in the vertical plane and is driven by a DC motor. A 
		spherical ball is placed on the link and can roll passively along its surface without direct actuation. The link's angular position is measured by an encoder, while a vision system estimates the ball's position in real time. This testbed is used to validate nonprehensile manipulation algorithms, including tasks such as transporting the ball between prescribed equilibrium points and sustaining stable periodic motions of the ball, see~\cite{Lynch1999}.

    \begin{figure}\begin{centering}
      \includegraphics[width=0.99\linewidth]{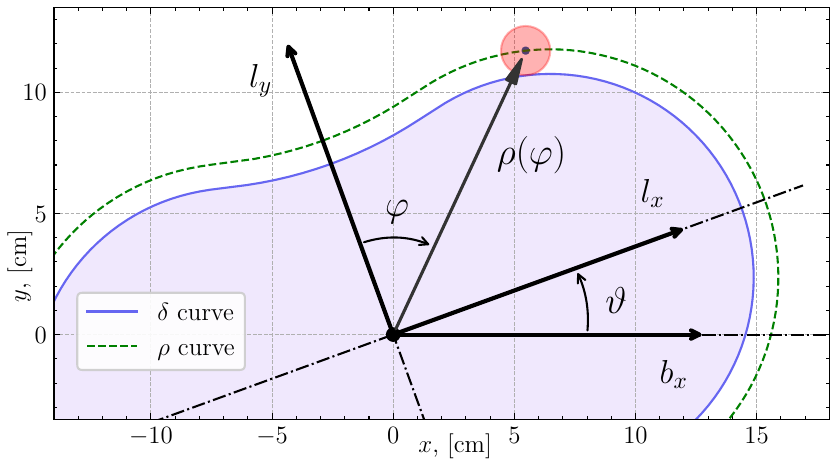}
      \par
      \caption{
        Kinematics of the Butterfly Robot. Definition of the generalized coordinates $q \equiv (\vartheta, \varphi)$.
      }
      \label{fig:butterfly-generalized-coordinates}
    \end{centering}\end{figure}

    Under standard modeling assumptions, including ideal ball rolling (see~\cite{Surov2015}), the equations of motion of the robot have the form
    \begin{equation}
    \label{eq:butterfly-dynamics}
      M(q)\ddot{q}+C\left(q,\dot{q}\right)\dot{q}+G(q)=F u,
    \end{equation}
    where the generalized coordinates $q\equiv\left(\vartheta,\varphi\right)$ define the angular position of the robot hand and position of the ball as illustrated in Fig.~\ref{fig:butterfly-generalized-coordinates}, the control signal $u$ represents the torque applied to the hand's revolute
    joint. 
    The matrix $M(q)\in\mathbb{R}^{2\times2}$ is invertible.
    The matrix $F \triangleq \left[1,0\right]^{\top}$ has a left annihilator $F_\perp \triangleq [0, 1]$.
    The exact expressions for the matrix functions $M(q),$ $C(q),$ $G(q)$ and the physical parameters of the robot are presented in~\cite{Surov2015}.

  \subsection{Periodic Trajectory}
    \begin{figure}\begin{centering}
      \includegraphics[width=0.99\linewidth]{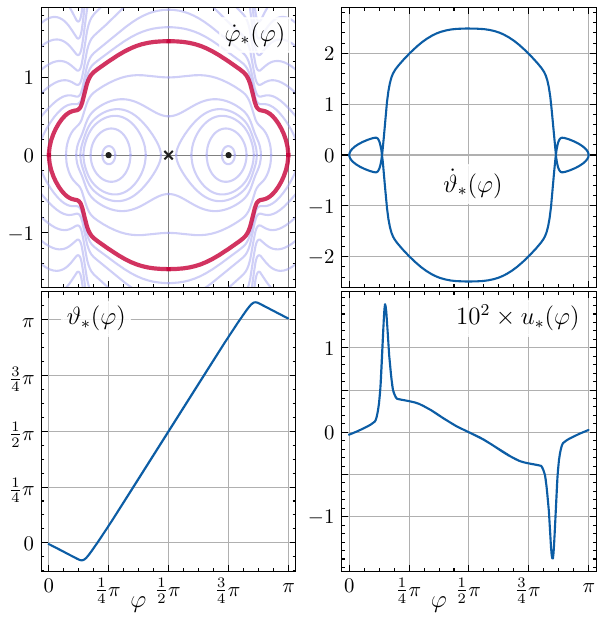}
      \par
      \caption{
        Phase portrait of the reduced dynamics; projections of the reference trajectory onto the phase planes; control signal along the reference trajectory.
      }
      \label{fig:butterfly-reference-trajectory}
    \end{centering}\end{figure}

    A periodic trajectory for the robot was obtained using the VHC approach, based on the results in~\cite{Shiriaev2006,LaHera2009,Surov2015}.
    Specifically, we define the VHC $h(q) \triangleq \vartheta - \Theta(\varphi)=0$, where $\Theta(\varphi)$ is a solution to the equation
    \[
      F_{\perp}G\left(\Theta(\varphi),\varphi\right)-c_{1}\sin(2\varphi)-c_{2}\sin(4\varphi)-c_{3}\left(\varphi-\frac{\pi}{2}\right) = 0
    \]
    and $c_1, c_2, c_3$ are the VHC parameters. With the VHC imposed, that is, $h(q) \equiv 0$, the robot dynamics reduce to the scalar differential equation
    \begin{equation}
    \label{eq:reduced-dynamics}
      \alpha\left(\varphi\right)\ddot{\varphi}+\beta\left(\varphi\right)\dot{\varphi}^{2}+\gamma\left(\varphi\right)=0
    \end{equation}
    with the coefficients
    \begin{align*}
      \alpha\left(\varphi\right) & = F_{\perp}M\left(\Theta\left(\varphi\right),\varphi\right)\left[\begin{array}{c}
      \Theta'\\
      1
      \end{array}\right], \,\,
      \gamma\left(\varphi\right) = F_{\perp}G\left(\Theta\left(\varphi\right),\varphi\right), \\
      \beta\left(\varphi\right) & = F_{\perp}M\left(\Theta\left(\varphi\right),\varphi\right)\left[\begin{array}{c}
      \Theta''\\
      0
      \end{array}\right]\\
       & \quad +F_{\perp}C\left(\Theta\left(\varphi\right),\varphi,\Theta'\left(\varphi\right),1\right)\left[\begin{array}{c}
      \Theta'\\
      1
      \end{array}\right].
    \end{align*}

    The locations and types of the equilibrium points, and hence the qualitative features of the phase portrait of this equation, are determined by the parameters $c_i$. Numerical analysis shows that for $c_{1}=0.008$, $c_{2}=-0.013$ and $c_{3}=0.010$, the equation~\eqref{eq:reduced-dynamics} admits low-speed periodic solutions, which are consistent with the assumption of ideal rolling. Specifically, the solution $\varphi_*(t)$ originating from $\varphi_*(0) = 0$, $\dot\varphi_*(0) = 0$ reaches the point $\varphi_*\left(\frac{T}{2}\right) = \pi$, $\dot\varphi_*\left(\frac{T}{2}\right) = 0$ and returns to its initial state at time $T=8.50$~s. The corresponding periodic trajectory of the robot, $q_*(t) \equiv \left[\Theta\left(\varphi_*(t)\right), \varphi_*(t)\right]^{\top}$, $\dot q_*(t) \equiv \left[\Theta'\left(\varphi_*(t)\right), 1\right]^{\top} \dot\varphi_*(t)$, is illustrated in Fig.~\ref{fig:butterfly-reference-trajectory}.
    
  \subsection{Transverse Coordinates}
    The local coordinates $(\tau,\xi)$ for the considering trajectory are defined as
    \begin{equation}
    \label{eq:xi12}
      \xi_{1} \triangleq \vartheta-\Theta\left(\varphi\right),\quad
      \xi_{2} \triangleq \dot{\vartheta}-\Theta'\left(\varphi\right)\dot{\varphi}.
    \end{equation}
    As seen in Fig.~\ref{fig:butterfly-reference-trajectory}, the projection of $x_{*}(t)$ onto the phase plane $\left(\varphi,\dot{\varphi}\right)$ is nearly elliptical. Accordingly, the phase variable $\tau$ is naturally introduced as the polar angle measured clockwise
    \begin{equation}
    \label{eq:tau}
      \tau = \tau\left(\varphi,\dot{\varphi}\right) \triangleq \mathrm{arctan2}\left(-\dot{\varphi},\varphi-\frac{\pi}{2}\right).
    \end{equation}
    Numerical evaluation shows that the projection variable, computed along the reference trajectory as $\tau_*(t) \triangleq \tau\bigl(x_*(t)\bigr),$ is strictly increasing on the interval $[0, T)$ and varies from $-\pi$ to $\pi$. This property allows $\tau_*(t)$ to be continuously extended beyond a single period according to~\eqref{eq:extending-tau}, with $T_\tau = 2\pi$.
    The resulting extension enables the reparametrization
    \[
      \phi\left(\tau\right) \triangleq \varphi_{*}\left(\tau_{*}^{-1}\left(\tau\right)\right),\quad
      \dot{\phi}\left(\tau\right) \triangleq \dot{\varphi}_{*}\left(\tau_{*}^{-1}\left(\tau\right)\right)
    \]
    and thereby the transverse coordinate
    \begin{equation}
    \label{eq:xi3}
      \xi_{3} \triangleq \big(\varphi-\phi\left(\tau\right)\big)\cos\tau-\big(\dot{\varphi}-\dot{\phi}\left(\tau\right)\big)\sin\tau.
    \end{equation}

    Numerical analysis confirms that the transformation $x\mapsto(\tau,\xi)$, defined by~(\ref{eq:xi12},\ref{eq:tau},\ref{eq:xi3}), satisfies the transverse coordinate requirements outlined in Section~\ref{sub:transverse-dynamics}.

  \subsection{Feedback Design}
    We obtained the linearized transverse dynamics of the robot
    \[
      \frac{d\xi}{d\tau} = A\left(\tau\right)\xi+B\left(\tau\right)w
    \]
    by following the procedure described in Section~\ref{sec:Orbital-Stabilization-Approach}, using the control transformation~\eqref{eq:control-transform} with parameters $\nu_{1}=15$ and $\nu_{2}=6$.
    The monodromy matrix $\Psi$ was computed by numerically integrating the matrix initial value problem
    \[
      \frac{dX}{d\tau} = A\left(\tau\right)X,\quad X\left(0\right)=I_{3\times3}
    \]
    over the interval $\tau \in [0, 2\pi]$, so that $\Psi=X\left(2\pi\right)$.
    This matrix has three distinct eigenvalues
    \[
      \mu_{1,2} = (-3.31 \pm 7.66i) \times 10^{-12}, \quad
      \mu_3 = 1.0,
    \]
    two of which are inside the unit circle. This is consistent with Proposition~\ref{prop:stable-subspace} and allows the identification of the vector $n(\tau)$ that defines the stable subspace $\mathcal{S}_\tau$. The computed vector is shown in Fig.~\ref{fig:stable-subspace}.
    As seen, the function $b(\tau) = n^{\top}\!\left(\tau\right)B\left(\tau\right)$ has two simple zeros 
    on the interval $\left[0,2\pi\right]$, occurring at $\tau_{z}=2.98$
    and $\tau_{z}=6.12$. 

    The computations above demonstrate that the transverse linearization of the robot dynamics satisfies the conditions of Theorem~\ref{thm:ltv-sliding-feedback}. This justifies the use of the control law~(\ref{eq:sliding-feedback}, \ref{eq:nonlinear-control}) for stabilization of the considering trajectory. The feedback gains were set to $k_1 = 8.0$ and $k_2 = 0.5$ and applied in experiments on the real robot.

    The results of the experimental validation are presented in Figs.~\ref{fig:transverse-transient} and~\ref{fig:phase-coordinates-transient}.
    The transverse coordinates $\xi$ remain bounded but do not completely vanish, which indicates the presence of parasitic dynamics, including sensor delays, unmodeled motor dynamics, and non-ideal ball rolling behavior.
    The sliding variable $s$ stays close to zero whenever $b(\tau) \ne 0$, but can diverge near $\tau_z$, where $b(\tau_z)$ vanishes. The control signal $u$ exhibits chattering, see \cite{Utkin1992}, when $s$ stays in the vicinity of zero, as it attempts to maintain the system on the sliding surface $\mathcal{S}_\tau$.

    The projection of the closed-loop system trajectory onto the $(\varphi,\vartheta)$ and $(\varphi,\dot\varphi)$ planes is shown in Fig.~\ref{fig:phase-coordinates-transient}. As observed, the virtual constraint is well tracked, except near the endpoints where the ball's speed is low and unmodeled frictional effects become more significant. The quality of velocity tracking is noticeably worse, which can be attributed to the presence of parasitic dynamics.

    An accompanying video of the experiment can be viewed at~\url{https://youtu.be/hx9mxDxHBB0}.

    \begin{figure}\begin{centering}
      \includegraphics[width=0.99\linewidth]{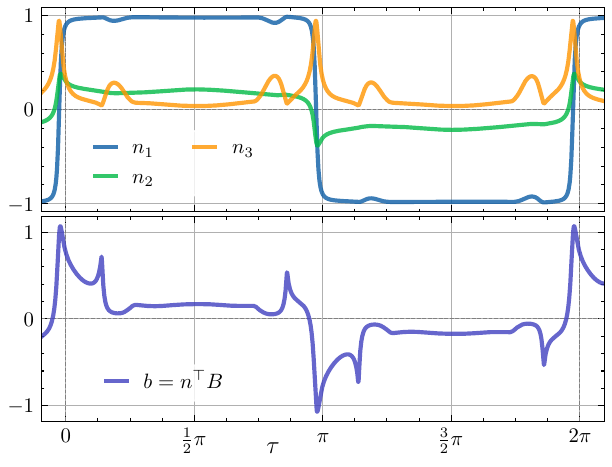}
      \par
      \caption{
        Components of the vector $n\left(\tau\right)$ defining the stable subspace $\mathcal{S}_{\tau}$, 
        and the coefficient $b(\tau) \equiv n^{\top}\!\left(\tau\right)B\left(\tau\right)$.
      }
      \label{fig:stable-subspace}
    \end{centering}\end{figure}
    
    \begin{figure}\begin{centering}
      \includegraphics[width=0.99\linewidth]{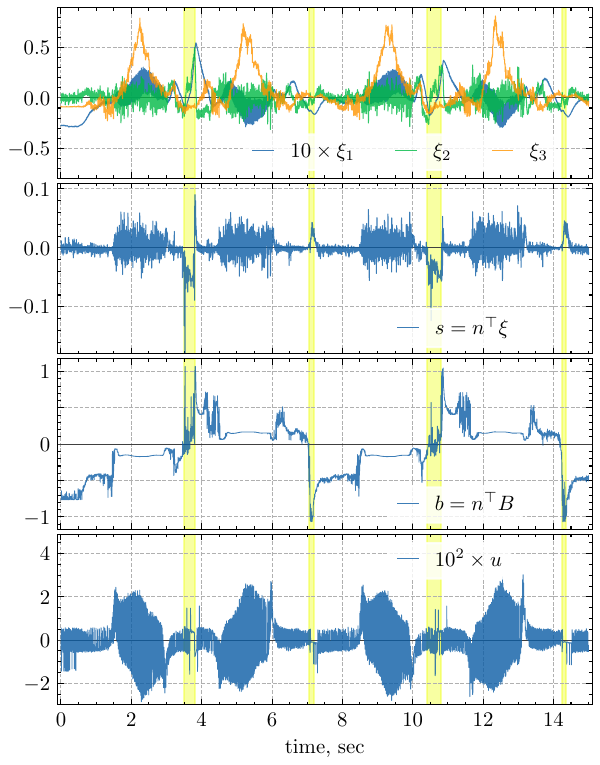}
      \par
      \caption{
        Transient behavior of the transverse coordinates $\xi$, sliding variable $s$, coefficient $b$, and control signal $u$. Regions where $b$ is close to zero -- and thus the sliding variable dynamics are uncontrollable -- are highlighted in yellow.
      }
      \label{fig:transverse-transient}
    \end{centering}\end{figure}
    
    \begin{figure}\begin{centering}
      \includegraphics[width=0.99\linewidth]{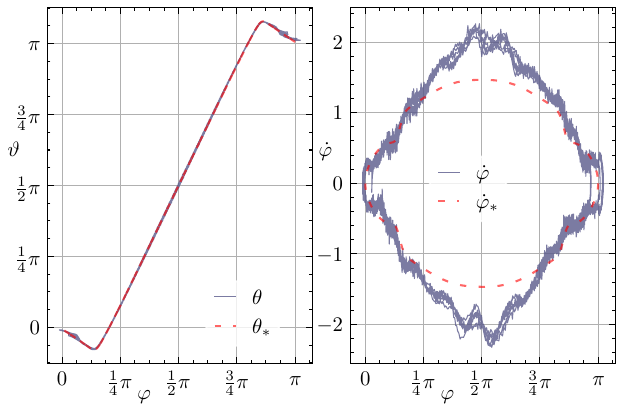}
      \par
      \caption{
        Projection of the closed-loop system trajectory, measured in a real experiment, onto the phase-space planes.
      }
      \label{fig:phase-coordinates-transient}
    \end{centering}\end{figure}

  \section{Concluding Remarks}
    This work introduced a new approach for stabilizing periodic orbits in mechanical systems with two degrees of freedom and one control input. The method combines subspace stabilization through partial feedback linearization, followed by PD-like feedback, with sliding-mode control. As with any sliding-mode-based feedback, the approach theoretically offers finite-time convergence of certain dynamics and improved tracking performance by compensating matched disturbances via high-frequency switching.

    In experiments, stronger feedback was limited by unilateral mechanical constraints, which restricted the maximum motor torque. Additional factors such as unmodeled dissipative forces, discretization effects, and errors in state estimation (due to vision and numerical differentiation) also affected performance. These limitations prevented achieving the ideal theoretical results. Nevertheless, the experiments confirm the feasibility of the approach.

    The method admits a natural extension to systems with $n$ degrees of freedom and $m = n - 1$ control inputs. This is particularly appealing, as it avoids solving computationally demanding LQR problems for high-dimensional LTV systems. On the other hand, the approach is less general, as it exploits specific structural properties of the dynamics. In this light, extending the approach to systems with $m < n-1$ control inputs remains an open question for future research.

  \appendix
  \section{Proof of Proposition~\ref{prop:stability-of-y}
  \label{sec:proof-stability-of-y}}
    Direct differentiation of
    $y\left(x\right)$ yields
    \[
      \dot{y}_{1}=\frac{\partial h}{\partial q}\dot{q}=L_{\dot q}h=y_{2},\qquad\dot{y}_{2}=
      \frac{\partial^2 h}{\partial q^2}\dot{q}^2+\frac{\partial h}{\partial q}\ddot{q}.
    \]
    Substituting the expression for $\ddot{q}$ from~\eqref{eq:mechanical-system}
    and control transform~\eqref{eq:control-transform} evaluated at
    $w=0$, we obtain
    \[
      \dot{y}=\left(\begin{array}{cc}
      0 & 1\\
      -\nu_{1} & -\nu_{2}
      \end{array}\right)y.
    \]
    Since $\nu_{1,2}>0$, the eigenvalues of the above matrix have strictly
    negative real parts. Therefore, the system is exponentially stable.
    \qed

  \section{Proof of Proposition~\ref{prop:stable-subspace}}
  \label{sec:appendix-prop-2}
    The main idea of the proof
    is to show that the function $V_{y}$, evaluated along solutions of
    the equation $\dot{x}=f\left(x\right)$, decreases after one period.
    This, in turn, imposes constraints on the Floquet multipliers of the
    linearized transverse dynamics. To make this argument explicit, we
    express the function $y$ in the transverse coordinates via
    \[
      \zeta\left(\xi,\tau\right) \triangleq y\left(x\left(\xi,\tau\right)\right).
    \]
    Because $y\left(x\right)=0$ for all $x\in x_{*}$ and $y\in C^{1}$,
    Hadamard’s lemma guarantees that $\zeta$ vanishes to first order
    along the orbit. Therefore, it admits the representation:
    \begin{align*}
      \zeta\left(\xi,\tau\right) &= J\left(\tau\right)\xi+O\left(\xi^{2}\right)\in\mathbb{R}^{2},\quad\text{where } \\
      J\left(\tau\right) & \triangleq \left.\frac{\partial y\left(x\right)}{\partial x}\right|_{x=x\left(0,\tau\right)}\left.\frac{\partial x\left(\xi,\tau\right)}{\partial\xi}\right|_{\xi=0}\in\mathbb{R}^{2\times3}.\nonumber 
    \end{align*}
    We claim, that $\mathrm{rank}\,J\left(\tau\right)=2$ for all $\tau$.
    Suppose, to the contrary, that for some $\tau$, $\mathrm{rank}\,J\left(\tau\right)\leq1$.
    Then there exist at least two linearly independent vectors $c_{1},$$c_{2}\in\mathbb{R}^{3}$,
    vanishing $J$. This can equivalently be written as 
    \[
      J\left(\tau\right)c_{i}=\left.\frac{\partial\zeta\left(\xi,\tau\right)}{\partial\left(\xi,\tau\right)}\right|_{\xi=0}\left(\begin{array}{c}
      c_{i}\\
      0
      \end{array}\right)=0.
    \]
    By the properties of $y$, we have $y\left(x\right)=0$ for all $x\in x_{*}$,
    and therefore $\left.\frac{\partial\zeta\left(\xi,\tau\right)}{\partial\tau}\right|_{\xi=0}=\frac{dy\left(x\left(0,\tau\right)\right)}{d\tau}=0.$
    Consequently, the vector $\left(0,0,0,1\right)^{\top}$ is also in
    the kernel of $\left.\frac{\partial\zeta\left(\xi,\tau\right)}{\partial\left(\xi,\tau\right)}\right|_{\xi=0}$.
    Since the transformation $x=x\left(\xi,\tau\right)$ is a diffeomorphism
    within $U_{\varepsilon}$, the matrix $\left.\frac{\partial x\left(\xi,\tau\right)}{\partial\left(\xi,\tau\right)}\right|_{\xi=0}$
    is of full rank four. Therefore, the above would imply that $\left.\frac{\partial y}{\partial x}\right|_{x=x\left(\tau,0\right)}$
    has at least three linearly independent right annihilators:
    \begin{align*}
      \left.\frac{\partial x}{\partial\left(\xi,\tau\right)}\right|_{\xi=0}\left(\begin{array}{c}
      c_{1}\\
      0
      \end{array}\right),\quad\left.\frac{\partial x}{\partial\left(\xi,\tau\right)}\right|_{\xi=0}\left(\begin{array}{c}
      c_{2}\\
      0
      \end{array}\right),\\
      \quad\left.\frac{\partial x}{\partial\left(\xi,\tau\right)}\right|_{\xi=0}\left(\begin{array}{c}
      0_{3}\\
      1
      \end{array}\right),
    \end{align*}
    which would force its rank to be at most one. However, by definition,
    \[
      \left.\frac{\partial y}{\partial x}\right|_{x=x\left(0,\tau\right)}=\left(\begin{array}{cc}
      \frac{\partial h(q)}{\partial q} & 0\\
      L_{\dot{q}}h(q) & \frac{\partial h(q)}{\partial q}
      \end{array}\right)_{x=x\left(0,\tau\right)},
    \]
    and the two rows are linearly independent since $\frac{\partial h(q)}{\partial q}\ne0$
    for all $q\in q_{*}$. This contradiction shows that $\mathrm{rank}J\left(\tau\right)=2$. 
    
    Next, use Proposition~\ref{prop:stability-of-y} in terms of variable
    $\zeta$. Consider a solution $x(t)$ of the equation $\dot{x}=f\left(x\right)$,
    originating in $x_{o}\in U_{\delta}$, which corresponds to $\tau_{o}=\tau\left(x_{o}\right)$
    and $\xi_{o}=\xi\left(x_{o}\right)$. For a sufficiently small $\delta>0$,
    the solution remains in $U_{\varepsilon}$ and within a finite interval
    $\Delta t$ achieves a point $x_{e}$, at which $\tau_{e}=\tau_{o}+T_{\tau}$
    and some $\xi=\xi_{e}$. In these points, the function $V_{y}$ is
    computed as 
    \begin{align}
    \label{eq:vy-for-ltv}
      V_{y}\left(x_{o}\right) &= \xi_{o}^{\top}J^{\top}PJ\xi_{o}+O\left(\xi_{o}^{3}\right),\nonumber \\
      V_{y}\left(x_{e}\right) &= \xi_{e}^{\top}J^{\top}PJ\xi_{e}+O\left(\xi_{e}^{3}\right),
    \end{align}
    where $J=J\left(\tau_{o}\right)=J\left(\tau_{e}\right)$, due to $x\left(0,\tau_{o}\right)=x\left(0,\tau_{o}+T_{\tau}\right)=x\left(0,\tau_{e}\right)$.
    The value $\xi_{e}$ is found from the monodromy matrix $\Psi$ of
    the linearized transverse dynamics $\frac{d\xi}{d\tau}=A\left(\tau\right)\xi$ as 
    \[
      \xi_{e}=\Psi\xi_{o}+O\left(\xi_{o}^{2}\right).
    \]
    From Proposition~\ref{prop:stability-of-y}, we obtain the relation 
    \begin{equation}
    \label{eq:inequlity-for-V}
      V_{y}\left(x_{e}\right) \leq \exp\left(-\alpha\Delta t\right)V_{y}\left(x_{o}\right),
    \end{equation}
    By the theorem on the continuous dependence of solutions on initial conditions, we have $\Delta t = T + O(\xi_o)$. 
    Since inequality~\eqref{eq:inequlity-for-V} holds for all $x_o$, including points arbitrarily close to the orbit (where $\xi_o$ is arbitrarily small), substituting the relations~\eqref{eq:vy-for-ltv} leads to the conclusion that the matrix $\Psi$ must satisfy
    \[
      \xi_{o}^{\top}\Psi^{\top}J^{\top}PJ\Psi\xi_{o}\leq\exp\left(-\alpha T\right)\xi_{o}^{\top}J^{\top}PJ\xi_{o} \quad \forall \, \xi_o \in \mathbb{R}^3.
    \]
    The matrices $\Psi^{\top}J^{\top}PJ\Psi$ and $J^{\top}PJ$ are symmetric and positive semi-definite. Therefore, the inequality can be extended to complex vectors as
    \[
      z^{\dagger}\Psi^{\top}J^{\top}PJ\Psi z\leq\exp\left(-\alpha T\right)z^{\dagger}J^{\top}PJz\quad\forall\,z\in\mathbb{C}^{3},
    \]
    where $z^{\dagger}$ denotes the conjugate transpose of $z$. Substituting $z$ with the eigenvectors $l_i$ of the matrix $\Psi$ and noting that $\Psi$ is nondefective, the following cases may occur:
    \begin{enumerate}
      \item 
        Vector $l_{i}$ lies in the kernel of $J$. In this case, $l_{i}$ is real, and the corresponding eigenvalue $\mu_{i}\in\mathbb{R}$. A complex $l_{i}$ would violate the nondefective condition. Since $\mathrm{rank}\,J=2$, this case can occur for at most one real eigenvector.
      \item 
        Vector $l_{i}$ is not in the kernel of $J$. Then 
        \[
          |\mu_{i}|^{2} \leq \exp\left(-\alpha T\right)<1\quad\Rightarrow\quad\left|\mu_{i}\right|<1.
        \]
        This case applies to at least two eigenvectors. Consequently, there are at least two stable Floquet multipliers, $\mu_{1}$ and $\mu_{2}$.
    \end{enumerate}
    \qed

 \section{Proof of Proposition~\ref{prop:definition-of-n}}
 \label{proof:definition-of-n}
    The existence and differentiability of $n(\tau)$ follow directly from the properties of the linearly independent solutions $e_1(\tau)$ and $e_2(\tau)$.

    Next, we show that any differentiable function $n(\tau)$ satisfying 
    \[
      n^\top\!(\tau) e_1(\tau) = 0, \, n^\top\!(\tau) e_2(\tau) = 0, \, n^\top\!(\tau) n(\tau) = 1 \, \forall \tau
    \]
    must satisfy equation~\eqref{eq:ode-for-n}. Let $f \triangleq \frac{dn}{d\tau}$. Differentiating $n^\top\!(\tau) e_i(\tau) = 0$ with respect to $\tau$ and substituting $\frac{de_i(\tau)}{d\tau} = A(\tau) e_i(\tau)$, we obtain 
    \[
      e^{\top}_i\!(\tau) f = -e^{\top}_i\!(\tau) A^{\top}\!(\tau)\,n(\tau).
    \]
    Hence, the orthogonal projection of $f$ onto the subspace spanned by $e_1, e_2$ is equal to $-e_i^{\top}\!(\tau) A^{\top}\!(\tau)\,n(\tau)$, and therefore $f$ admits the decomposition
    \[
      f = -A^{\top}\!(\tau)\, n(\tau) + a \, n(\tau),
    \]
    for some scalar function $a$. Imposing $n^\top\!(\tau) n(\tau) = 1$ and differentiating yields $a = n^{\top}(\tau)A^{\top}(\tau)n(\tau)$. Substituting this expression into the decomposition above, we obtain
    \[
      \frac{dn}{d\tau} = -A^{\top}\!(\tau)n + (n^{\top}A^{\top}\!(\tau)n) n = -\left(I-nn^{\top}\right)A^{\top}\!(\tau)n,
    \]
    as claimed. \qed

  \bibliographystyle{unsrt}
  \bibliography{refs}
\end{document}